%% file: main.tex
\documentclass[10pt,twocolumn,letterpaper]{article}

\usepackage{cvpr}

\definecolor{FHblue}{RGB}{0,92,184}
\definecolor{FHbg}{RGB}{245,248,255}
\newcommand{\kw}[1]{\textcolor{FHblue}{\textbf{#1}}}      
\newcommand{\hi}[1]{\colorbox{FHbg}{\strut #1}}   
\usepackage{amsmath,amssymb,amsthm}

\usepackage[table,svgnames]{xcolor}
\usepackage{booktabs,multirow,makecell,colortbl}

\usepackage{arydshln}  
\definecolor{Acol}{RGB}{70,130,180}      
\definecolor{ABcol}{RGB}{46,139,87}      
\definecolor{ABCcol}{RGB}{218,165,32}    
\newcommand{\A}[1]{\cellcolor{Acol!10}{#1}}
\newcommand{\AB}[1]{\cellcolor{ABcol!12}{#1}}

\usepackage[export]{adjustbox}
\usepackage{graphicx}
\usepackage{algorithm}
\usepackage{algpseudocode}

\newtheorem{theorem}{Theorem}[section]

\theoremstyle{definition}

\theoremstyle{remark}

\input{preamble}

\definecolor{cvprblue}{rgb}{0.21,0.49,0.74}
\usepackage[pagebackref,breaklinks,colorlinks,allcolors=cvprblue]{hyperref}

\def\paperID{15754} 
\def\confName{CVPR}
\def\confYear{2026}

\title{FedHarmony: Harmonizing Heterogeneous Label Correlations in \\Federated Multi-Label Learning}

\author{
Zhiqiang Kou\textsuperscript{1,2}\thanks{Equal contribution.} \quad
Junxiang Wu\textsuperscript{1,2}\footnotemark[1] \quad
Wenke Huang\textsuperscript{3} \quad
Wenwen He\textsuperscript{3} \quad
Ming-Kun Xie\textsuperscript{4} \\
Changwei Wang\textsuperscript{5} \quad
Yuheng Jia\textsuperscript{1,2}\thanks{Corresponding Author.} \quad
Di Jiang\textsuperscript{6} \quad
Yang Liu\textsuperscript{6} \quad
Xin Geng\textsuperscript{1,2 †} \quad
Qiang Yang\textsuperscript{6} \\
\textsuperscript{1}School of Computer Science and Engineering, Southeast University, Nanjing, China \\
\textsuperscript{2} Key Laboratory of New Generation Artificial Intelligence Technology and Its Interdisciplinary \\ Applications (Southeast University), Ministry of Education, China \\
\textsuperscript{3}Wuhan University, China \quad
\textsuperscript{4}RIKEN Center for Advanced Intelligence Project, Japan \\
\textsuperscript{5}Qilu University of Technology (Shandong Academy of Sciences), China \\
\textsuperscript{6}Academy for Artificial Intelligence, Hong Kong Polytechnic University, Hong Kong, China \\
{\tt\small yhjia@seu.edu.cn  \quad  xgeng@seu.edu.cn}
}

\begin{document}
\maketitle
\input{sec/0_abstract}    
\input{sec/1_intro}

\input{sec/2_method}
\input{sec/3_exp}

{
    \small
    \bibliographystyle{ieeenat_fullname}
    \bibliography{main}
}

\end{document}

%% file: preamble.tex
\usepackage{arydshln}



%% file: sec/0_abstract.tex
\begin{abstract}

Federated Multi-Label Learning is a distributed paradigm where multiple clients possess heterogeneous multi-label data and perform collaborative learning under privacy constraints without sharing raw data. However, modeling label correlations under heterogeneous distributions remains challenging. Due to client-specific label spaces and varying co-occurrence patterns, correlations learned by individual clients inevitably deviate from the global structure, a phenomenon we term \textit{label correlation drift}.
To address this, we propose \textit{FedHarmony}, a framework that harmonizes heterogeneous label correlations across clients. It introduces \textit{consensus correlation}, capturing agreement among other clients and serving as a global teacher to correct biased local estimates.
During aggregation, FedHarmony evaluates each client by both data size and correlation quality, assigning weights accordingly. Moreover, we develop an accelerated optimization algorithm for FedHarmony and theoretically establish faster convergence without sacrificing accuracy.  Experiments on real-world federated multi-label datasets show that FedHarmony consistently outperforms state-of-the-art methods.

\end{abstract}

%% file: sec/1_intro.tex
\section{Introduction}
\label{sec:intro}

Multi-Label Learning (MLL) \cite{zhang2013review, zhang2014lift} aims to predict multiple labels for each instance \cite{huang2012multi, kouijcaio, NEURIPS2022_751ef1e7,koubldl, kou2025nips}. 
For example, an image may contain \textit{building}, \textit{street}, and \textit{person}, and the goal is to identify all relevant labels for unseen samples. 
MLL has shown strong performance in various applications, including visual recognition \cite{cvintro}, sentiment analysis \cite{emointro}, and medical image understanding \cite{medicintro}. A central challenge in MLL is modeling label correlations, such as common co-occurrence patterns \cite{zhu2017multi, xiao2024dual}. 
Recent methods use Graph Convolutional Networks \cite{chen2019multi, zhu2017multi} and Transformer-based architectures \cite{xiecounterfactual, Yang_2023_ICCV} to explicitly encode these label relationships, which significantly improves prediction performance.

\begin{figure*}[!ht]
  \centering
  \includegraphics[width=0.8\textwidth]{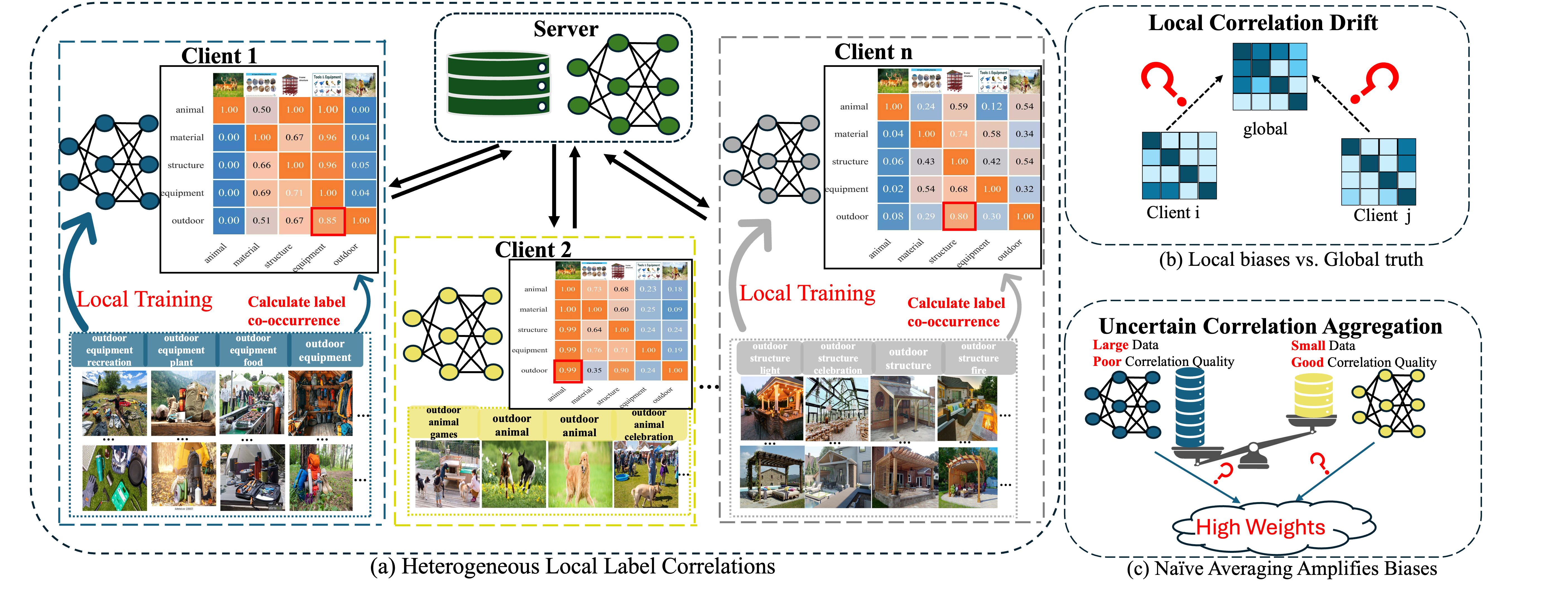}
 \caption{Label co-occurrence patterns on the FLAIR dataset~\cite{song2022flair}. 
Different clients exhibit distinct label co-occurrence frequencies, 
which deviate from the global pattern representing the true data distribution.}
  \label{dongji}
\end{figure*}

Driven by growing privacy and security concerns \cite{privacy,huangwen1, huangwen2, qi1}, multi-label learning is increasingly conducted under the Federated Learning framework \cite{qi2, meng2024federated,qi3}, as illustrated in Fig. \ref{dongji}(a), where each client holds its own multi-label dataset and Federated Multi-Label Learning (FedMLL) \cite{flmll1,guo2023privacy,flmll3} refers to a distributed paradigm in which multiple clients with multi-label data perform collaborative learning under privacy constraints without sharing raw data.
As discussed above,its key objective is to align and integrate the label correlations learned from distributed clients , enabling the central server to recover the overall label dependency structure  from decentralized data  \cite{recoverintro,recoverintro2, koutnnls, koutcsvt}.

However, achieving this goal is non-trivial   under heterogeneous data distributions \cite{fldatadistri,flconcept}.
As shown in Fig. \ref{dongji}(a), we visualize the label co-occurrence matrices on the FLAIR dataset\footnote{This is a multi-labeled federated learning dataset.} \cite{song2022flair}, showing that each client exhibits distinct co-occurrence frequencies. For example, the co-occurrence between the labels \textit{outdoor} and \textit{equipment} is very high in Client~1, while it is much lower in Client~2.
Moreover, existing methods \cite{liu2024fedlgt, mcmahan2017fedavg, li2020fedprox} weight clients by their training data size and average their updates accordingly, ignoring the quality of the learned label correlations.  As shown in Fig. \ref{dongji}(c), a client with a large amount of data may still learn poor correlation structures, yet it receives a disproportionately high aggregation weight, raising the question of whether this is desirable.

To address these challenges, we propose \textit{FedHarmony}, a federated multi-label learning framework that harmonizes heterogeneous label correlations across clients.
In FedMLL, each client observes only a subset of the full label space, which inevitably leads to biased local correlation estimates \cite{biasdata3.1}.
We argue that no single client can fully capture the true label relationships; however, the correlations that are consistent across most clients are more likely to reflect the underlying global semantics. Based on this insight, we introduce the notion of \textit{consensus correlation}, defined for each client as the collective agreement of all other clients. During local training, this consensus correlation serves as a global teacher: it provides structural guidance that continuously corrects the client’s biased local view, enabling its learned label relationships to gradually align with the global consensus.
In the aggregation stage, FedHarmony further considers the quality of each client’s learned label structure. To improve training efficiency, we also develop an accelerated optimization algorithm that speeds up convergence without compromising predictive performance.
Our main contributions are summarized as follows

\begin{itemize}[leftmargin=12pt]
    \item 
   We provide the first systematic study of label correlation drift in FedMLL,  and we propose FedHarmony to explicitly address this issue through consensus-guided correlation modeling.

    \item 
   We develop an accelerated optimization algorithm for FedHarmony and theoretically show that it achieves faster convergence while preserving model accuracy.

    \item 
   Extensive experiments on multiple federated multi-label benchmarks demonstrate that FedHarmony consistently outperforms existing state-of-the-art methods.
\end{itemize}

%% file: sec/2_method.tex
\section{FedHarmony}

$\textbf{Notation}.$  We consider a federated multi-label learning system with $K$ clients, indexed by $k$. 
Each client $k$ owns a private dataset $\mathcal{D}_k = \{x_i, y_i\}_{i=1}^{N_k}$, where $N_k$ denotes the private data number for $k$-th client. The input space is $\mathcal{X} \subset \mathbb{R}^d$ and the label space is $\mathcal{Y} = \{0,1\}^C$, where $C$ is the number of labels. For an instance $x \in \mathcal{X}$, its label vector is $y = (y_1,\dots,y_C)^\top \in \{0,1\}^C$, where $y_c = 1$ indicates that label $c$ is relevant to $x$. 
We denote the global model parameter at the beginning of the $t$-th communication round as $w^{t}$. 
The server then broadcasts $w^{t}$ to each client, initializing the local model as $w_k^{t} \leftarrow w^{t}$. 
Each client performs local optimization on its private dataset and updates the model via stochastic gradient descent: $w_k^{t} \leftarrow w_k^{t} - \eta \sum_{i \in \mathcal{B}_k} \ell\big(w_k^{t}, \xi_i^{(k)}\big)$, 
where $\mathcal{B}_k$ denotes a mini-batch sampled from the local dataset $\mathcal{D}_k$, 
$\xi$ is the query instance, 
and $\eta$ is the local learning rate.
After local updates, each client uploads its parameter $w_k^{t}$ to the server,
which performs weighted aggregation to obtain the next global model: $w^{t+1} = \sum_{k=1}^{K} \alpha_k\, w_k^{t},$ where $\alpha_k$ is the aggregation weight for client $k$. Following \cite{xie2023class}, we use the binary cross-entropy loss  for multi-label classification.

\noindent
\fcolorbox{FHblue}{FHbg}{%
  \parbox{\dimexpr\linewidth-2\fboxsep-2\fboxrule}{%
    \textbf{Overview of FedHarmony.}\;
    In federated multi-label learning, client-specific label spaces and heterogeneous co-occurrence patterns cause locally estimated label correlations to become inconsistent, leading to \kw{label correlation drift}. 
    FedHarmony aims to recover a \kw{consensus label structure} that captures shared global semantics while correcting biased local correlations through consensus-guided updates (Sec \ref{1} and Sec \ref{2}). \hi{Moreover}, clients differ in how well they learn these structures; thus, aggregation cannot rely solely on data size and must also account for \kw{\emph{learning quality}} to prevent unreliable correlations from dominating the global model (Sec \ref{3}).\vspace{2pt}
  }%
}

\subsection{Consensus Label Correlation Teacher}
\label{1}

In federated multi-label learning \cite{flmll1,guo2023privacy,flmll3}, data heterogeneity causes the locally learned label correlations to become biased \cite{biasdata3.1}. We argue that no single client knows the true label relationships, and that correlations shared by the majority are more likely to reflect the truth \cite{align3.1}. 
Thus, we introduce a \textit{consensus correlation}, representing the collective agreement of all clients except the target one. During client-side training, we use the consensus label correlation as a teacher to prevent the locally learned correlations from drifting away from the global structure.  Next, we define the consensus label correlation:

At communication round $t$, client $k$ obtains prediction scores for its local dataset 
$\mathcal{D}_k=\{x_i\}_{i=1}^{N_k}$ using its current model $f_k(\cdot;\theta_k^{(t)})$: $F_k^{(t)}=\big[f_k(x_1;\theta_k^{(t)}),\ldots,f_k(x_{N_k};\theta_k^{(t)})\big]^\top
\in[0,1]^{N_k\times C}$. These scores serve as soft estimates of label occurrence.  
From them, we compute the empirical marginal and joint probabilities:
\begin{equation}
\hat p_{k,c}^{(t)}
=\frac{1}{N_k}\sum_{i=1}^{N_k} F_{k,ic}^{(t)}, 
\qquad
\hat p_{k,cc'}^{(t)}
=\frac{1}{N_k}\sum_{i=1}^{N_k} F_{k,ic}^{(t)}F_{k,ic'}^{(t)} ,
\end{equation}
where $\hat p_{k,c}^{(t)}$ estimates the marginal probability that label $c$ is predicted to appear on client $k$, and $\hat p_{k,cc'}^{(t)}$ estimates the joint probability that labels $c$ and $c'$ are predicted to appear simultaneously. To capture their dependency strength, we compute a phi-style correlation coefficient:
\begin{equation}
R_{k,cc'}^{(t)}
=\frac{\hat p_{k,cc'}^{(t)}-\hat p_{k,c}^{(t)}\hat p_{k,c'}^{(t)}}
{\sqrt{\hat p_{k,c}^{(t)}(1-\hat p_{k,c}^{(t)})\,\hat p_{k,c'}^{(t)}(1-\hat p_{k,c'}^{(t)})}+\varepsilon},
\end{equation}
where $\varepsilon>0$ is a small constant for numerical stability.
For the $k$-th client, the local label correlation matrix is defined as $R_k^{(t)}=\big[R_{k,cc'}^{(t)}\big]_{c,c'=1}^{C}\in\mathbb{R}^{C\times C}$.

For client $k$ at round $t$, the consensus label correlation is defined as the leave-one-out
population consensus constructed from the \emph{label--label} correlation statistics of all other clients:
\begin{equation}
\label{eq:expert_consensus_arrow}
\underbrace{\{\mathbf{R}_j^{(t)}\}_{j\neq k}}_{\text{all clients except }k}
\;\xrightarrow{\ \mathcal{A}_t\ }\;
\mathbf{R}_{\mathrm{exp},k}^{*(t)} \in \mathcal{C}.
\end{equation}
Here, $\{\mathbf{R}_j^{(t)}\}_{j\neq k}$ are the uploaded correlation matrices from all clients except $k$,
and $\mathcal{A}_t(\cdot)$ denotes the consolidation operator that aggregates them into a
consensus structure.

\subsection{Consensus-Guided Correction}
\label{2}

We correct the local correlation by aligning it to the round-wise expert consensus. 
For client $i$ at round $t$, let $\mathbf{R}_i^{(t)}$ be its local label--label correlation 
and $\mathbf{R}_{\mathrm{exp},i}^{*(t)}$ the expert consensus (Sec.~\S2.1). 
The correction is imposed via a correlation-alignment loss
\begin{equation}
\label{eq:align_full}
\mathcal{L}^{\mathrm{align}}_{i,t}
\;=\;
\lambda\,\Psi\!\big(\mathbf{R}_i^{(t)},\,\mathbf{R}_{\mathrm{exp},i}^{*(t)}\big),
\end{equation}
where $\Psi(\cdot,\cdot)$ is a fixed distance/divergence in correlation space (kept consistent throughout).

\begin{figure}[!h]
  \centering
  \includegraphics[width=1\linewidth]{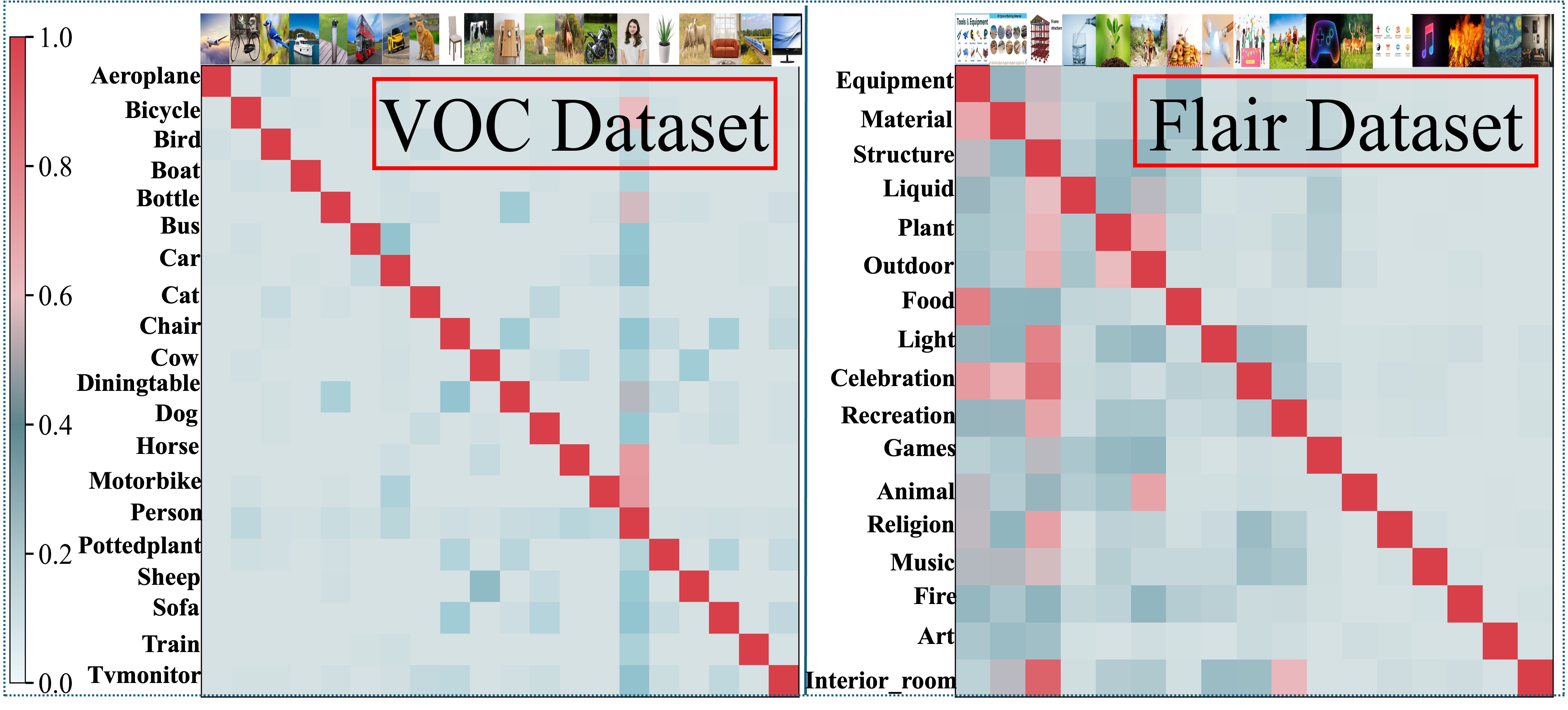}
   \caption{Label correlation matrix on the VOC and FLAIR datasets.
Each matrix visualizes the empirical co-occurrence probability between every pair of labels. Labels listed on the horizontal and vertical axes correspond to the semantic categories in each dataset, and brighter colors indicate higher co-occurrence probability.   }
  \label{xishu}
\end{figure}

However, not all labels are mutually correlated: each label interacts with only a small subset of others.
As shown in Fig.~\ref{xishu}, the label correlation matrix  is \emph{sparse} and approximately \emph{block-structured} \cite{block-structured}.
If we partition it into $g$ smaller (approximately) rank-$r$ submatrices and align correlations only \emph{within} clusters \cite{cluster3.2}, we can both improve optimization efficiency and incur negligible information loss.
We now provide theory from these two perspectives and state two theorems with their proofs.

Let $\mathbf{R}\in\mathbb{S}^C$ be a client correlation and $\mathbf{R}_{\mathrm{exp}}^{*}$ the (round-wise) expert consensus.
Define the weighted Frobenius loss $f(\mathbf{R})=\|\mathbf{\Gamma}\!\circ\!(\mathbf{R}-\mathbf{R}_{\mathrm{exp}}^{*})\|_F^2$ with weights $\mathbf{\Gamma}=(\gamma_{cc'})\ge 0$.
Let $\{\mathcal{S}_g\}_{g=1}^G$ be a label partition, and let
$\mathcal{U}=\{\mathbf{X}\in\mathbb{S}^C:\mathrm{supp}(\mathbf{X})\subseteq \cup_g \mathcal{S}_g^2\}$ denote the block-diagonal subspace,
with orthogonal projector $P_{\mathcal{U}}$.
Define
\begin{equation}
\left\{
\begin{aligned}
\gamma_{\mathrm{in}} \; &:=\; \min_{(c,c')\in \cup_g \mathcal{S}_g^2} \gamma_{cc'} \;>\; 0,\\
\gamma_{\mathrm{out}} \; &:=\; \max_{(c,c')\notin \cup_g \mathcal{S}_g^2} \gamma_{cc'} \;\in\; [0,\rho\, \gamma_{\mathrm{in}}),\ \ \rho<1,
\end{aligned}
\right.
\end{equation}
here, $\gamma_{\mathrm{in}}$ and $\gamma_{\mathrm{out}}$ denote the in-cluster and
cross-cluster correlation weights, respectively. Decompose the consensus as 
\begin{equation}
\left\{
\begin{aligned}
\mathbf{B} &:= P_{\mathcal{U}}\,\mathbf{R}_{\mathrm{exp}}^{*},\\
\mathbf{E} &:= \mathbf{R}_{\mathrm{exp}}^{*}-\mathbf{B},
\end{aligned}
\right.
\end{equation}
here, $\mathbf{B}$ denotes the in-cluster part of $\mathbf{R}_{\mathrm{exp}}^{*}$, while
$\mathbf{E}$ contains all cross-cluster correlations.

\begin{theorem}
\label{thm1}
Consider gradient descent with stepsize $\eta\le 1/(2\max_{c,c'} \gamma_{cc'}^{2})$ on
(i) the full objective $f_{\mathrm{full}}(\mathbf{R})=\|\mathbf{\Gamma}\!\circ\!(\mathbf{R}-\mathbf{R}_{\mathrm{exp}}^{*})\|_F^2$ and
(ii) the block-restricted objective $f_{\mathrm{blk}}(\mathbf{R})=\|\mathbf{\Gamma}\!\circ\!(P_{\mathcal{U}}\mathbf{R}-\mathbf{B})\|_F^2$.
From the same initialization $\mathbf{R}^{(0)}$, after $T$ steps,
\begin{equation}
\left\{
\begin{aligned}
f_{\mathrm{blk}}\!\big(\mathbf{R}_{\mathrm{blk}}^{(T)}\big) &\le (1-2\eta\, \gamma_{\mathrm{in}}^{2})^{T}\, f_{\mathrm{blk}}\!\big(\mathbf{R}^{(0)}\big),\\
f_{\mathrm{full}}\!\big(\mathbf{R}_{\mathrm{full}}^{(T)}\big) &\le (1-2\eta\, \gamma_{\mathrm{out}}^{2})^{T}\, f_{\mathrm{full}}\!\big(\mathbf{R}^{(0)}\big).
\end{aligned}
\right.
\end{equation}
Since $\gamma_{\mathrm{in}}\gg \gamma_{\mathrm{out}}$, block-wise alignment enjoys a strictly faster linear convergence rate.
\end{theorem}

\begin{proof}
Vectorize under the Frobenius inner product. Both objectives are quadratic with diagonal Hessians:
$H_{\text{full}}=2\,\mathrm{diag}(\gamma_{cc'}^2)$ and
$H_{\text{blk}}=2\,\mathrm{diag}(\gamma_{cc'}^2\mathbf{1}_{(c,c')\in \cup_g \mathcal{S}_g^2})$.
Thus the strong convexity constants satisfy
$\mu_{\text{full}}=2\min \gamma_{cc'}^2\le 2\gamma_{\mathrm{out}}^2$ and
$\mu_{\text{blk}}=2\min_{\text{in}} \gamma_{cc'}^2\ge 2\gamma_{\mathrm{in}}^2$.
With $\eta\le 1/L$, standard results for strongly convex quadratics yield the linear factors $(1-\eta\mu)^T$.
\end{proof}

\begin{theorem}[]
\label{thm2}
Let $\mathbf{R}_{\mathrm{blk}}^{(T)}$ be the iterate produced by minimizing $f_{\mathrm{blk}}$.
Then
\[
f_{\mathrm{full}}\!\big(\mathbf{R}_{\mathrm{blk}}^{(T)}\big)
=
\underbrace{f_{\mathrm{blk}}\!\big(\mathbf{R}_{\mathrm{blk}}^{(T)}\big)}_{\text{in-block alignment error}}
+
\underbrace{\|\,\mathbf{\Gamma}_{\mathrm{out}}\!\circ \mathbf{E}\,\|_F^{2}}_{\text{upper bound}},
\]
where $\mathbf{\Gamma}_{\mathrm{out}}$ keeps only cross-block weights.
In particular, if $\mathbf{E}=0$ or $\gamma_{\mathrm{out}}$ is small, the second term is at most $O(\gamma_{\mathrm{out}}^{2}\|\mathbf{E}\|_F^{2})$ and negligible.
\end{theorem}

\begin{proof}
Since $\mathbf{R}_{\mathrm{blk}}^{(T)}\in\mathcal{U}$ and $\mathbf{B}=P_{\mathcal{U}}\mathbf{R}^*$, the supports of
$\mathbf{\Gamma}\!\circ(\mathbf{R}_{\mathrm{blk}}^{(T)}-\mathbf{B})$ (in-block) and $\mathbf{\Gamma}\!\circ \mathbf{E}$ (cross-block) are disjoint and orthogonal under the Frobenius inner product.
Hence
$\|\mathbf{\Gamma}\!\circ(\mathbf{R}_{\mathrm{blk}}^{(T)}-\mathbf{R}^*)\|_F^2
=\|\mathbf{\Gamma}\!\circ(\mathbf{R}_{\mathrm{blk}}^{(T)}-\mathbf{B})\|_F^2+\|\mathbf{\Gamma}\!\circ \mathbf{E}\|_F^2$,
and the second term equals $\|\mathbf{\Gamma}_{\mathrm{out}}\!\circ \mathbf{E}\|_F^2$.
\end{proof}

By Theorem~\ref{thm1}, restricting correction to in-cluster pairs increases curvature
(from $\gamma_{\mathrm{out}}$ to $\gamma_{\mathrm{in}}$) and yields a faster linear rate.
By Theorem~\ref{thm2}, ignoring cross-cluster entries incurs at most
$\|\mathbf{\Gamma}_{\mathrm{out}}\!\circ \mathbf{E}\|_F^{2}$ additional loss, i.e., negligible when the consensus is
near block-diagonal or cross-cluster weights are small. Motivated by Theorems~\ref{thm1} and~\ref{thm2}, we correct only the
\emph{high-correlation} subset of label pairs in the full label space. Concretely, we partition
labels using the expert correlation $R_{\mathrm{exp},i}^{*(t)}$ to obtain $g$ clusters
$\{\mathcal{S}_g\}_{g=1}^{G}$\footnote{See Algorithm~\ref{alg:spectral} for clustering details.},
so that strongly correlated labels fall into the same group and alignment focuses on dense,
high-signal subspaces. Given $\{\mathcal{S}_g\}_{g=1}^{G}$, Eq.~(\ref{eq:align_full})
is replaced by the within-cluster objective

\begin{equation}
\label{eq:align_blockwise}
\mathcal{L}^{\mathrm{align}}_{i,t}
=\lambda \sum_{g=1}^{G}
\Psi\!\Big(R_i^{(t)}[\mathcal{S}_g,\mathcal{S}_g],\,
          R_{\mathrm{exp},i}^{*(t)}[\mathcal{S}_g,\mathcal{S}_g]\Big).
\end{equation}

\begin{algorithm}[!t]
\caption{Spectral clustering on expert correlation}
\label{alg:spectral}
\begin{algorithmic}[1]
\Require Expert correlation $\mathbf{R}_{\mathrm{exp}}^{*(t)}\in\mathbb{R}^{C\times C}$, number of clusters $G$
\State \textbf{Affinity:} $\mathbf{S} \gets \frac{1}{2}\big(|\mathbf{R}_{\mathrm{exp}}^{*(t)}|+|\mathbf{R}_{\mathrm{exp}}^{*(t)}|^{\top}\big)$; set $\mathrm{diag}(\mathbf{S})\leftarrow 0$
\State \textbf{Laplacian:} $\mathbf{D}\gets \mathrm{diag}(\mathbf{S}\mathbf{1})$, \; $\mathbf{L}\gets \mathbf{D}^{-1/2}(\mathbf{D}-\mathbf{S})\mathbf{D}^{-1/2}$
\State \textbf{Embedding:} take $\mathbf{U}\in\mathbb{R}^{C\times G}$ as eigenvectors of $\mathbf{L}$ for the $G$ smallest eigenvalues
\State \textbf{Row normalization:} $\mathbf{U}_{c,:}\leftarrow \mathbf{U}_{c,:}/\|\mathbf{U}_{c,:}\|_{2}$ for all $c$
\State \textbf{Clustering:} run $k$-means on $\{\mathbf{U}_{c,:}\}_{c=1}^{C}$ to obtain clusters $\{\mathcal{S}_g\}_{g=1}^{G}$
\State \Return $\{\mathcal{S}_g\}_{g=1}^{G}$
\end{algorithmic}
\end{algorithm}

\subsection{Correlation-Aware Aggregation}
\label{3}

We aggregate client updates by jointly considering (i) data quantity and (ii) the \emph{learning quality} of each client's label-structure modeling. Early rounds emphasize quantity since local correlations are unreliable; later rounds increasingly favor structural quality.  For client $i$ at round $t$, define its block-wise structural discrepancy
\begin{equation}
s_i^{(t)} \;=\; \sum_{g=1}^{G}
\Psi\!\Big(\mathbf{R}_i^{(t)}[\mathcal{S}_g,\mathcal{S}_g],\,
          \mathbf{R}_{\mathrm{exp},i}^{*(t)}[\mathcal{S}_g,\mathcal{S}_g]\Big),
\end{equation}
and map it to a quality score (larger is better) via a monotone decreasing transform, e.g.
$q_i^{(t)} \;=\; \exp\!\big(-\gamma\, s_i^{(t)}\big), \qquad \gamma>0$.
Normalize counts and qualities: $\bar n_i \;=\; \frac{n_i}{\sum_j n_j}$, $\bar q_i^{(t)} \;=\; \frac{q_i^{(t)}}{\sum_j q_j^{(t)}}$.
Let $\alpha^{(t)}\!\in[0,1]$ decrease over rounds to shift the aggregation from quantity-driven to structure-driven, e.g.
$\alpha^{(t)} \;=\; \max\!\big(0,\, 1 - t/T_0\big)$, 
with a user-chosen transition horizon $T_0$. The per-round aggregation weight is
\begin{equation}
w_i^{(t)} \;=\; \alpha^{(t)}\,\bar n_i \;+\; \big(1-\alpha^{(t)}\big)\,\bar q_i^{(t)},
\qquad
\sum_i w_i^{(t)}=1.
\end{equation}
The server aggregates model parameters as
$\boldsymbol{\theta}^{(t+1)} \;=\; \sum_{i} w_i^{(t)}\,\boldsymbol{\theta}_i^{(t+1)}$.
When $t$ is small, $\alpha^{(t)}\!\approx\!1$ and the rule reduces to quantity-weighted averaging (FedAvg-like), which stabilizes early training.
As $t$ grows, $\alpha^{(t)}\!\downarrow\!0$ and clients with better correlation alignment (larger $q_i^{(t)}$) dominate aggregation.

%% file: sec/3_exp.tex
\section{Experiment}

\subsection{Experimental Setup}

\newcommand{\ourrow}{\rowcolor{RoyalBlue!8}}
\newcommand{\ours}{\textbf{\textsc{FedHarmony}}} 
\newcommand{\best}[1]{\textbf{#1}}               

\begin{table*}[!th]
\centering
\caption{Overall comparisons on three multi-label benchmarks under non-IID settings. 
Rows for \ours\ are lightly shaded; bold numbers denote the best within each dataset block.}

\setlength{\tabcolsep}{3pt} 

\renewcommand{\arraystretch}{0.85} %

\begin{tabular}{l l||cc|ccc|ccc} 

\hline\hline

\rowcolor{gray!25} 
&Method & mAP (\%) & O\_mAP (\%) & CP (\%) & CR (\%) & CF1 (\%) & OP (\%) & OR (\%) & OF1 (\%) \\
\hline
\multirow{8}{*}{\textcolor{gray!70}{\textit{FLAIR}}} & FedAvg      & 35.4 & 70.6 & 34.1 & 27.9 & 30.7 & 70.2 & 56.0 & 62.3 \\
 &FedCurv     & 35.4 & 71.6 & 34.9 & 25.5 & 29.5 & 72.7 & 53.2 & 61.4 \\
 &FedProx     & 39.6 & 75.7 & 40.2 & 30.0 & 34.3 & 75.7 & 58.2 & 65.8 \\
 &FedNova     & 30.3 & 55.6 & 12.3 & 11.0 & 11.6 & 56.5 & 35.9 & 43.9 \\
 &FedLGT      & 36.9 & 72.8 & 39.2 & 25.0 & 30.5 & 75.5 & 51.4 & 61.1 \\
 &SphereFed   & 35.2 & 71.9 & 35.1 & 22.5 & 27.5 & 77.2 & 48.3 & 59.4 \\
 &FedRDN      & 35.5 & 71.4 & 35.0 & 26.3 & 30.0 & 72.3 & 54.1 & 61.9 \\
 \hdashline
\ourrow 
 &\ours       & \best{51.0} & \best{84.0} & \best{49.3} & \best{43.6} & \best{46.3} & \best{79.0} & \best{71.5} & \best{75.1} \\

\hdashline
\multirow{8}{*}{\textcolor{gray!70}{\textit{COCO-80}}}& FedAvg      & 63.4 & 73.9 & 63.2 & 54.9 & 58.7 & 69.3 & 64.4 & 66.7 \\
 &FedCurv     & 63.3 & 73.3 & 64.0 & 53.7 & 58.4 & 68.6 & 64.4 & 66.4 \\

 &FedProx     & 64.0 & 74.1 & 62.7 & 56.2 & 59.3 & 68.5 & 65.8 & 67.1 \\
 &FedNova     &  4.3 &  4.9 &  2.5 & 49.0 &  4.8 &  4.2 & 57.8 &  7.9 \\
 &FedLGT      & 64.5 & 74.4 & 61.4 & 57.5 & 59.4 & 67.5 & 66.8 & 67.2 \\
 &SphereFed   & 63.0 & 72.7 & 61.8 & 55.2 & 58.3 & 65.9 & 65.3 & 65.6 \\
 &FedRDN      & 63.4 & 73.6 & 64.0 & 54.5 & 58.9 & 69.1 & 64.4 & 66.6 \\
 \hdashline
\ourrow
 &\ours       & \best{71.4} & \best{79.9} & \best{73.5} & \best{58.6} & \best{65.2} & \best{78.6} & \best{67.6} & \best{72.7} \\

\hdashline

 \multirow{8}{*}{\textcolor{gray!70}{\textit{VOC2007}}}&FedAvg      & 75.7 & 69.2 & 81.8 & 54.5 & 65.4 & 61.6 & 63.9 & 62.7 \\
 &FedCurv     & 72.8 & 56.9 & 75.5 & 50.7 & 60.7 & 62.4 & 57.5 & 59.9 \\
 &FedProx     & 75.5 & 68.9 & 82.6 & 53.2 & 64.7 & 67.1 & 61.4 & 64.1 \\
 &FedNova     & 18.1 & 36.6 &  4.7 &  1.9 &  2.7 & \best{94.9} & 11.0 & 19.7 \\
 &FedLGT      & 78.1 & 70.8 & 77.7 & 61.0 & 68.3 & 60.2 & 68.0 & 63.9 \\
 &SphereFed   & 75.2 & 59.4 & 81.5 & 56.8 & 66.9 & 58.3 & 65.8 & 61.8 \\
 &FedRDN      & 78.3 & 72.2 & 79.9 & 57.2 & 66.6 & 70.6 & 65.0 & 67.7 \\
 \hdashline
\ourrow
 &\ours       & \best{86.9} & \best{89.1} & \best{88.5} & \best{71.2} & \best{78.9} & 88.5 & \best{78.8} & \best{83.4} \\
\hline\hline
\end{tabular}
\label{duibishiyan} 
\end{table*}

\begin{figure*}[!h]
  \centering
  
  \includegraphics[width=0.88\textwidth]{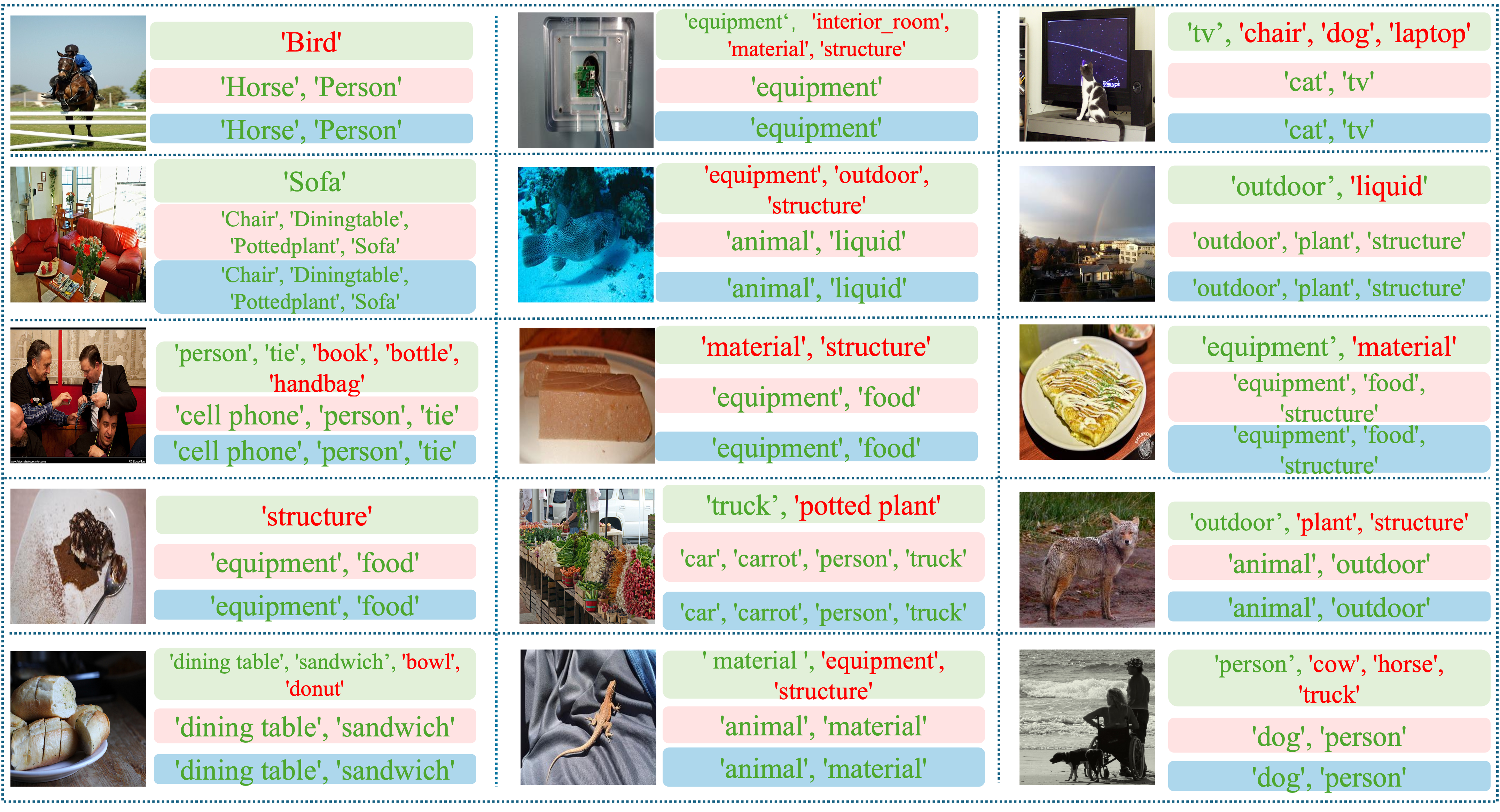}
  \caption{\textbf{Qualitative comparison.} For each image, we show the predictions of 
\textit{FedAvg} ({\color{LightGreen!25}\rule{6pt}{6pt}}) and 
\textit{Ours} ({\color{cyan!25}\rule{6pt}{6pt}}), alongside the 
ground-truth labels ({\color{pink!25}\rule{6pt}{6pt}}). 
Our method suppresses spurious labels (e.g., \textit{'bird'} on COCO horse) and recovers missing semantics (e.g., \textit{equipment}, \textit{material}, \textit{structure}), demonstrating better correlation-aware recognition across scenes.}
  \label{fig:qual_keshihua}
\end{figure*}

$\mathbf{Datasets.}$ Following \citet{liu2024fedlgt}, we evaluate on three multi-label datasets:
\emph{FLAIR}~\citep{song2022flair}, \emph{MS\mbox{-}COCO}~\citep{lin2014coco}, and \emph{PASCAL VOC}~\citep{everingham2015pascal}.
\emph{FLAIR} contains \emph{429{,}078} images from \emph{51{,}414} Flickr users with a two-level label space (\emph{17} coarse and \emph{1{,}628} fine labels), naturally exhibiting quantity and label-distribution skew.
\emph{MS\mbox{-}COCO} provides approximately \emph{330K} images annotated with \emph{80} object categories in the standard multi-label formulation.
\emph{PASCAL VOC} offers \emph{20} categories; we report results on \emph{VOC2007} (9{,}963 images) and \emph{VOC2012} train/val (11{,}540 images).  To convert MS-COCO and PASCAL VOC into heterogeneous federated datasets, we induce non-IIDness along three axes.  The full protocol (Dirichlet concentration, client counts, transform sets, and seeds) is detailed in the Appendix.

$\mathbf{Counterparts.}$
We compare FedHarmony with representative federated learning methods from three major paradigms:
(i) optimization-driven aggregation methods, including FedAvg~\citep{mcmahan2017fedavg}, FedProx~\citep{li2020fedprox}, and FedNova~\citep{wang2020fednova};
(ii) curvature- or geometry-aware optimization methods, such as FedCurv~\citep{shoham2019fedcurv} and SphereFed~\citep{dong2022spherefed};
and (iii) task- or feature-aware baselines, including FedLGT for federated multi-label classification~\citep{liu2024fedlgt} and FedRDN for mitigating distribution skew through feature-level augmentation~\citep{jin2025fedrdn}.
All methods are implemented under identical experimental settings, sharing the same backbone, optimizer, and communication budget to ensure a fair comparison.

$\mathbf{Implementation}$  $\mathbf{Details.}$
Unless otherwise stated, we follow the settings of \citet{liu2024fedlgt, huangwen3, huang4}. 
All methods use a ViT-B/16 backbone \cite{dosovitskiy2021vit} with a $C$-way sigmoid classification head. 
Local training runs for 5 epochs per round using Adam \cite{kingma2015adam} with a learning rate of $10^{-4}$ and batch size $16$, and we set the total communication rounds to $T=50$. 
To account for quantity skew (e.g., in FLAIR), we adopt non-uniform client sampling \cite{li2020fedprox}, where the probability of selecting client $i$ is proportional to its local data size. 
All methods share identical backbones, optimizers, schedules, and client partitions to ensure fairness. 
Experiments are implemented in PyTorch and trained on eight NVIDIA RTX~4090 GPUs. 
Additional details are provided in the Appendix.

 $\mathbf{Evaluation}$   $\mathbf{Metrics. }$We evaluate multi-label prediction performance using eight widely adopted metrics~\cite{zhu2017multi}. 
Specifically, we report mean Average Precision (mAP, ↑), overall mean Average Precision (O-mAP, ↑), class-wise precision (CP, ↑), class-wise recall (CR, ↑), class-wise F1 score (CF1, ↑), overall precision (OP, ↑), overall recall (OR, ↑), and overall F1 score (OF1, ↑).
Higher values indicate better performance for all metrics.

\subsection{Comprehensive Comparative Analysis}
Table~\ref{duibishiyan} and Fig.~\ref{fig:qual_keshihua} present both quantitative and qualitative comparisons under heterogeneous non-IID settings across COCO-80, VOC2007, and FLAIR. 
FedHarmony achieves the best performance on all datasets and metrics, with a notable gain of more than 10 mAP on FLAIR over the strongest baseline, and similarly strong margins on COCO-80 and VOC2007. 
Qualitative results further show that FedAvg often yields spurious or incomplete predictions, whereas FedHarmony produces more consistent and semantically accurate label sets.

From the results above, we obtain the following observations:
\begin{itemize}[leftmargin=1.2em]
    \item 
    FedHarmony substantially outperforms FedAvg and FedProx, indicating that aligning cross-client correlation structures is more effective than parameter averaging for multi-label representation learning.

    \item
    FedCurv and SphereFed show marginal improvements, suggesting that curvature correction is insufficient to address the semantic inconsistencies introduced by heterogeneous label dependencies.

    \item 
    Compared with FedLGT and FedRDN, FedHarmony remains stable under severe non-IID conditions, demonstrating that correlation-structure alignment provides a more robust and generalizable inductive bias.
\end{itemize}

Overall, these results verify that harmonizing client-specific label correlations leads to globally coherent and locally adaptive representations, yielding consistent performance gains across federated multi-label benchmarks.

\begin{table}[t]
\centering \small
\caption{Cumulative training time comparison on FLAIR and VOC2007 datasets across communication rounds. 
The proposed \textbf{Block-Optimized (B-OPT)} method notably accelerates convergence compared to the \textbf{No-Block-Optimized (NB-OPT)} variant.}
\label{tab:training_time}
\setlength{\tabcolsep}{9pt}
\renewcommand{\arraystretch}{1.1}
\begin{tabular}{c|cc|cc}
\hline\hline
\multirow{2}{*}{\textbf{Round}} & \multicolumn{2}{c|}{\textbf{Flair}} & \multicolumn{2}{c}{\textbf{VOC2007}} \\
\cline{2-3} \cline{4-5} 
 & \cellcolor{gray!25}\textbf{B-OPT} & \textbf{NB-OPT} & 
   \cellcolor{gray!25}\textbf{B-OPT} & \textbf{NB-OPT} \\
\hline
1  & \cellcolor{gray!15}03:59 & 05:14 & \cellcolor{gray!15}01:32 & 02:29 \\
2  & \cellcolor{gray!15}07:55 & 10:15 & \cellcolor{gray!15}02:36 & 04:03 \\
3  & \cellcolor{gray!15}12:05 & 15:43 & \cellcolor{gray!15}03:42 & 05:39 \\
4  & \cellcolor{gray!15}16:09 & 21:06 & \cellcolor{gray!15}04:49 & 07:17 \\
5  & \cellcolor{gray!15}20:14 & 27:14 & \cellcolor{gray!15}05:55 & 08:56 \\
6  & \cellcolor{gray!15}24:17 & 33:31 & \cellcolor{gray!15}07:05 & 10:34 \\
7  & \cellcolor{gray!15}28:18 & 39:53 & \cellcolor{gray!15}08:14 & 12:08 \\
8  & \cellcolor{gray!15}32:17 & 46:26 & \cellcolor{gray!15}09:29 & 13:51 \\
9  & \cellcolor{gray!15}36:21 & 51:58 & \cellcolor{gray!15}10:32 & 15:29 \\
10 & \cellcolor{gray!15}40:22 & 56:19 & \cellcolor{gray!15}11:43 & 17:09 \\
\hline\hline
\end{tabular}%
\end{table}

\begin{table}[!h] 
\centering 
\caption{Comparison of No-Block vs. Block-Optimized across three benchmarks. 
All metrics are reported in percentage (\%).}
\setlength{\tabcolsep}{3pt}
\begin{tabular}{l l||ccc} 
\hline\hline

\rowcolor{gray!25} 
 & Method & mAP $\uparrow$ & O\_mAP  $\uparrow$ & CP $\uparrow$ \\
\hline
\multirow{2}{*}{\textcolor{gray!70}{\textit{COCO-80}}}
& No-Block        & 71.2 & 79.8 & 73.8 \\ 
& Block-Optimized & 72.4 & 80.8 & 71.9 \\ 
\hdashline
\multirow{2}{*}{\textcolor{gray!70}{\textit{VOC2007}}}
& No-Block        & 86.1 & 88.9 & 87.0 \\ 
& Block-Optimized & 87.0 & 89.1 & 87.4 \\ 
\hdashline
\multirow{2}{*}{\textcolor{gray!70}{\textit{FLAIR}}}
& No-Block        & 47.0 & 81.9 & 45.6 \\ 
& Block-Optimized & 48.0 & 83.8 & 45.9 \\ 
\hline\hline
\end{tabular}
\label{block}
\end{table}

 \definecolor{rowgray}{RGB}{248,250,253}
 \definecolor{colhead}{RGB}{235,240,245}
 \definecolor{accent}{RGB}{40,85,150}

\begin{table}[!h]
\centering
\caption{\textbf{Wilcoxon signed-rank test results} between No-Block and Block-Optimized on three benchmarks. 
All $p$-values $>$ 0.05, indicating no significant differences.}
\setlength{\tabcolsep}{10pt}
\renewcommand{\arraystretch}{1.25}
\rowcolors{2}{rowgray}{white}
\begin{tabular}{l c c c}
\hline \hline
\rowcolor{gray!25} 
\textbf{Dataset} & \textbf{W} & \textbf{$p$-value} & \textbf{Result} \\
\midrule
COCO-80  & 7.0 & 0.382 & \textbf{\textcolor{accent}{tie (not significant)}} \\
VOC2007  & 0.0 & 0.148 & \textbf{\textcolor{accent}{tie (not significant)}} \\
FLAIR    & 2.0 & 0.547 & \textbf{\textcolor{accent}{tie (not significant)}} \\
\hline \hline
\end{tabular}
\label{wilcoxon_results}
\end{table}

\begin{figure*}[!h]
  \centering
  \includegraphics[width=1\linewidth]{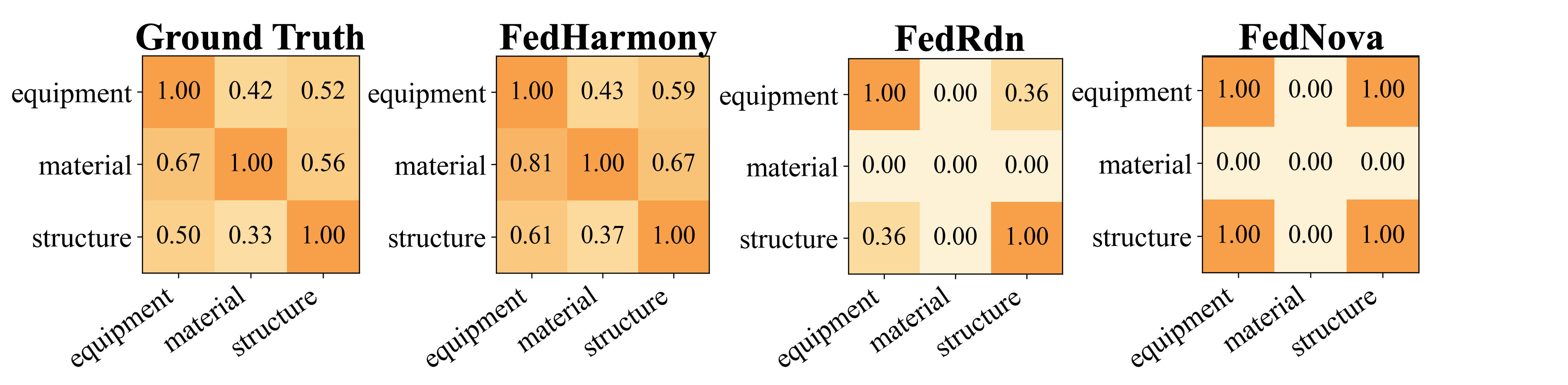}
  \caption{Qualitative comparison of learned label correlation structures on the FLAIR dataset.}
  \label{fig:flair_correlation} 
\end{figure*}

\subsection{Impact of Block-wise Optimization}
We evaluate the effect of block-wise optimization by comparing the No-Block and Block-Optimized variants on COCO-80, VOC2007, and FLAIR (Table~\ref{block}). The performance differences across all eight metrics remain minimal, typically within 0.3\%–0.5\%.
To assess statistical significance, we conduct a Wilcoxon signed-rank test using paired metric values for each dataset. As shown in Table~\ref{wilcoxon_results}, all $p$-values exceed 0.05 (COCO-80: 0.382; VOC2007: 0.148; FLAIR: 0.547), indicating no significant difference between the two configurations.
We further compare training efficiency (Table~\ref{tab:training_time}). The Block-Optimized variant substantially reduces cumulative training time, achieving a 28.3\% reduction on FLAIR (56:19 → 40:22) and 31.7\% on VOC2007 (17:09 → 11:43) by Round 10.
Overall, block-wise optimization introduces no statistically significant performance degradation while providing clear efficiency gains, confirming that the proposed clustering strategy preserves accuracy and improves computational efficiency.

\subsection{The remaining experiments}

$\textbf{Qualitative Analysis of Correlation Structure.}$
We compare the label correlation matrices learned by FedHarmony and competing methods against the ground-truth (GT) matrix. As shown in Fig.~\ref{fig:flair_correlation}, FedHarmony reconstructs a structure that closely matches the GT and accurately recovers subtle co-occurrence patterns, such as the moderate \textit{equipment–material} relation (0.43 vs.\ GT 0.42). In contrast, methods such as FedRdn and FedNova produce sparse matrices that collapse most correlations toward zero, failing to capture the underlying semantic dependencies.
These qualitative results support our central claim. Conventional aggregation strategies lose structural information and suffer from uncertain correlation aggregation, whereas FedHarmony, through its Correlation Expert and structure-driven aggregation, harmonizes heterogeneous client structures and recovers a globally faithful correlation matrix.

\begin{table}[t] 
\centering

\small

\caption{Ablation on COCO-80 and FLAIR benchmarks. \textbf{Base} = FedAvg. 
\textcolor{Acol}{\rule{6pt}{6pt}} \textbf{A} = ECL;
\textcolor{ABcol}{\rule{6pt}{6pt}} \textbf{A+B} = ECL + CAA.
(Legend text slightly shortened for space). All metrics are reported in percentage (\%).
}

\setlength{\tabcolsep}{5pt} 
\renewcommand{\arraystretch}{0.9}

\begin{tabular}{l ccc ccc} 
\toprule \toprule

\multirow{2}{*}{\makecell[l]{Metric}} &
\multicolumn{3}{c}{\textbf{COCO-80}} & 
\multicolumn{3}{c}{\textbf{FLAIR}} \\ 
\cmidrule(lr){2-4}\cmidrule(lr){5-7} 
& Base & \A{+A} & \AB{+A+B} 
& Base & \A{+A} & \AB{+A+B} \\ 

\midrule

mAP $\uparrow$  & 63.4 & \A{69.5} & \AB{71.2} & 35.4 & \A{46.4} & \AB{47.0} \\
O\_mAP $\uparrow$ & 73.9 & \A{78.5} & \AB{79.8} & 70.6 & \A{79.6} & \AB{81.9} \\
CP $\uparrow$   & 63.2 & \A{73.0} & \AB{73.8} & 34.1 & \A{44.2} & \AB{45.6} \\
CR $\uparrow$   & 54.9 & \A{54.4} & \AB{58.4} & 27.9 & \A{40.9} & \AB{39.3} \\
CF1 $\uparrow$  & 58.7 & \A{62.4} & \AB{65.2} & 30.7 & \A{42.5} & \AB{42.2} \\
OP $\uparrow$   & 69.3 & \A{79.2} & \AB{78.5} & 70.2 & \A{72.8} & \AB{74.8} \\
OR $\uparrow$   & 64.4 & \A{64.5} & \AB{67.7} & 56.0 & \A{73.3} & \AB{73.9} \\
OF1 $\uparrow$  & 66.7 & \A{71.1} & \AB{72.7} & 62.3 & \A{73.0} & \AB{74.3} \\
\toprule  \toprule 
\end{tabular}
\label{tab:ablation_cool} 
\end{table}

$\textbf{Ablation Study.}$
To verify the contribution of each component in our framework, we conduct ablations on three datasets. The primary results are summarized in Table~\ref{tab:ablation_cool} and the results for VOC2007 are deferred to the Appendix. \emph{Base} = FedAvg; \textbf{A} = expert-guided correlation loss; \textbf{B} = correlation-aware aggregation; \textbf{C} = block-wise clustering.

 One factor is added at a time (Base $\rightarrow$ A $\rightarrow$ A+B $\rightarrow$ A+B+C); all other settings are fixed.  From the Table~\ref{tab:ablation_cool}, we can get the below conclusion

\begin{figure*}[!h]
  \centering
  \includegraphics[width=1\linewidth]{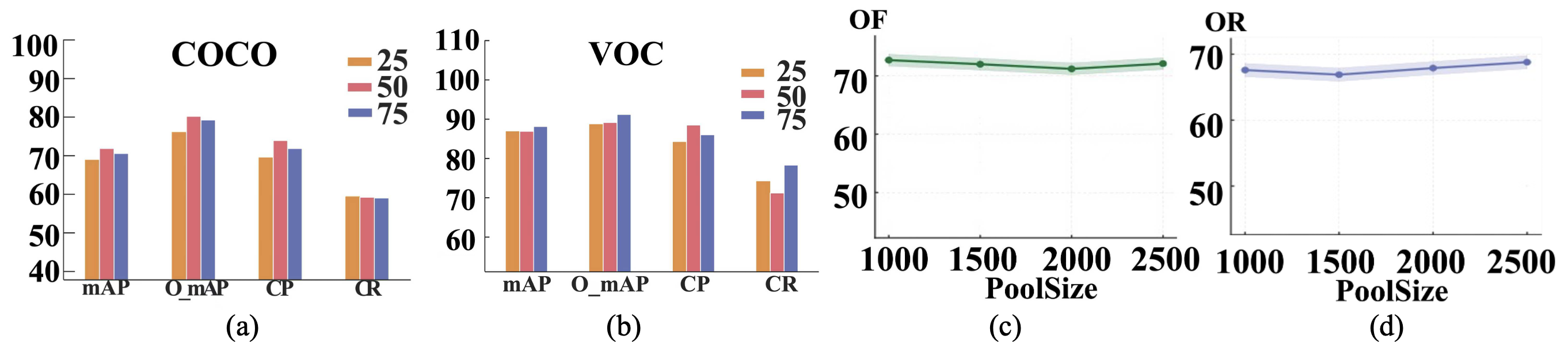}
 \caption{Performance analysis based on client configurations: (a, b) impact of the number of participating clients, and (c, d) impact of the total number of clients.}
\label{fig:combined_client_analysis} 
\end{figure*}

\begin{figure}[!h]
  \centering
  
  \includegraphics[width=0.45\textwidth]{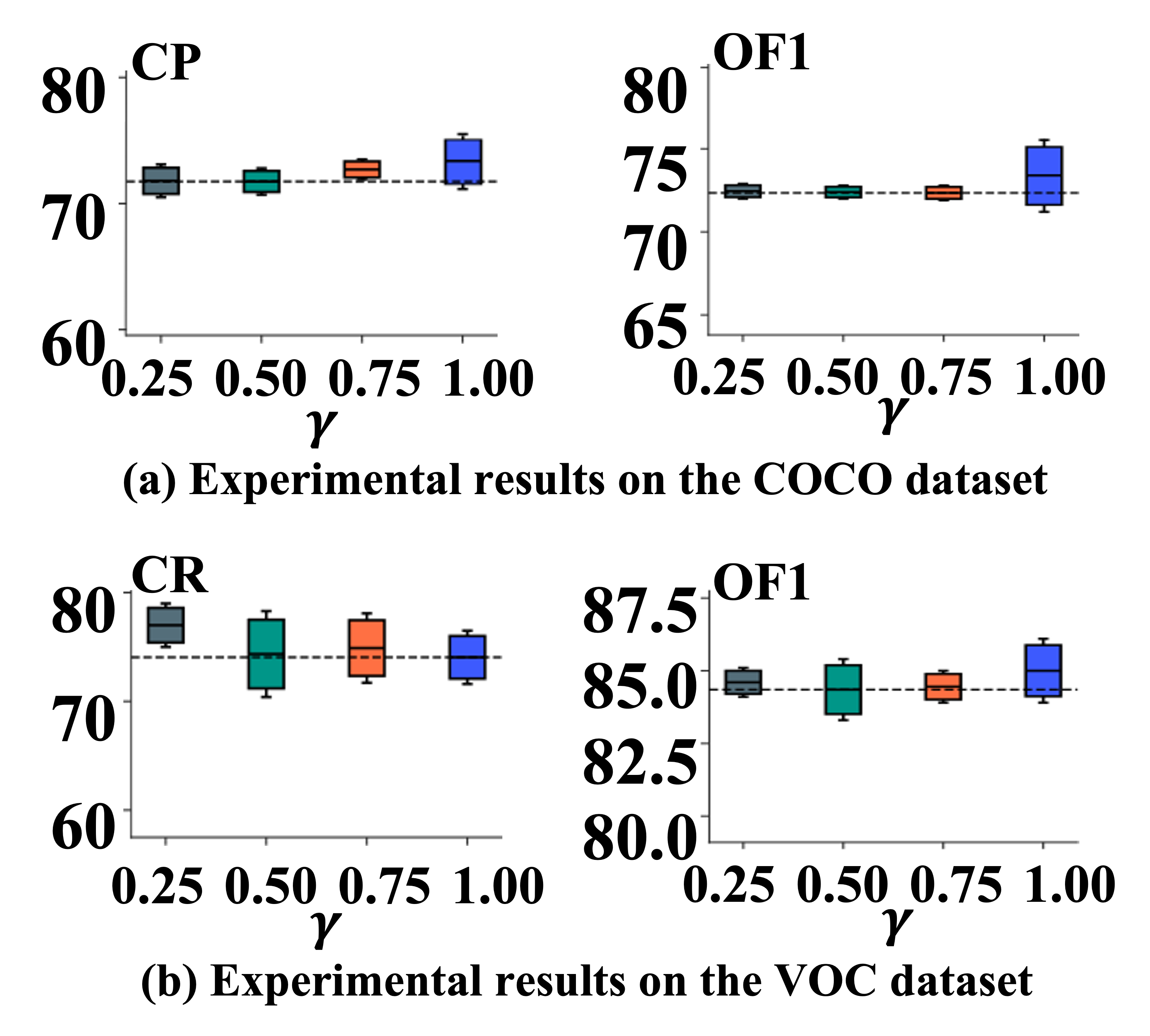}
  \caption{Robustness of FedHarmony under different label skew levels controlled by $\gamma$.}
  \label{fig:gamma_robustness}
\end{figure}
\begin{itemize}[leftmargin=1.2em,itemsep=1pt,topsep=2pt]
\item \textbf{A vs. Base.} Primary improvements on all datasets for mAP and O\_mAP; recall increases where non-IID is strongest.
\item \textbf{A+B vs. A.} Consistent gains in OR/OF1 across datasets; a minor precision drop (OP) on COCO.
\item \textbf{A+B+C vs. A+B.} Further increases in CR/OF1; effect largest on the sparsest dataset (FLAIR); slight CP/OP decrease on COCO.
\item \textbf{Overall.} A contributes most of the total mAP gain; B adds recall-oriented improvements via quality-weighted aggregation; C improves alignment conditioning. The full configuration (A+B+C) is best across metrics.
\end{itemize}
$\textbf{Robustness to  Structural Heterogeneity.}$
We assess robustness to label imbalance by varying the Dirichlet concentration parameter $\gamma$, which controls the FL-MLL sampling process. Smaller $\gamma$ produces stronger non-IID partitions, where each client observes only a small subset of labels (details in Appendix~A).
Fig. \ref{fig:gamma_robustness} shows that as $\gamma$ decreases from 1.0 to 0.25, FedHarmony maintains almost unchanged mAP, OF1, CP, and CR on COCO-80 and VOC2007, while all baselines degrade rapidly. This stability stems from aligning client-specific correlation structures toward a shared consensus rather than relying on raw sample frequencies, preventing over-represented labels from dominating aggregation.
Overall, the flat curves across different $\gamma$ values confirm that FedHarmony generalizes reliably under severe label skew and remains structurally robust across heterogeneous clients.

$\textbf{Influence of the Number of Participating Clients.}$
We evaluate the impact of client participation by testing three participation ratios (25\%, 50\%, 75\%) on COCO-80, VOC2007, and FLAIR. As shown in Fig.~\ref{fig:combined_client_analysis}(a,b), performance remains highly stable across different ratios, with variations typically within 1–2\%. This indicates that the framework maintains consistent accuracy even when only a subset of clients joins each communication round.
These results highlight the robustness of our correlation-aware aggregation and block-wise optimization, which effectively reduce performance fluctuations caused by client sampling and allow the global model to generalize reliably under heterogeneous and partially available client settings.

$\textbf{Effect of the Number of Clients.}$
We evaluate scalability by varying the number of participating clients from 1000 to 2500 on the COCO and VOC benchmarks while keeping the total training data fixed (Fig.~\ref{fig:combined_client_analysis}(c,d)).
 Across both datasets, FedHarmony exhibits highly stable performance, with Overall Recall and Overall F1 fluctuating by only 1–2 points as the federation size increases. This indicates that the method consistently harmonizes label correlations even under increasingly fragmented client distributions. The slight performance drop at extremely large scales (e.g., 2500 clients on VOC) stems from heightened local data sparsity and stronger correlation heterogeneity~\cite{clientdrift}.
Overall, the results demonstrate strong robustness to client scaling and confirm FedHarmony’s ability to maintain semantic consistency in large-scale federated multi-label learning.

\section{Conclusion}

We propose FedHarmony, which explicitly addresses local correlation drift and enforces cross-client consistency through expert-guided correlation regularization and correlation-aware aggregation. To further enhance optimization stability, we develop a lightweight block-wise acceleration scheme that improves convergence efficiency while preserving accuracy and training stability.
Extensive experiments with statistical significance validation on COCO-80, VOC2007, and FLAIR demonstrate that our method consistently achieves superior and robust performance under heterogeneous data distributions  settings.

\section*{Acknowledgement}
\label{sec:acknowledgement}

This research was supported by the Jiangsu Science Foundation (BG2024036, BK20243012), the National Science Foundation of China (62125602, U24A20324, 92464301), the New Cornerstone Science Foundation through the XPLORER PRIZE, and the Fundamental Research Funds for the Central Universities (2242025K30024 supported by the National). This work was Natural Science Foundation of China under Grant U24A20322 and Grant 62576094. This research work is also supported by the Big Data Computing Center of Southeast University.